\documentclass[10pt,twocolumn,letterpaper]{article}

\usepackage{cvpr} % uncomment this for the final submission
\usepackage{times}
\usepackage{epsfig}
\usepackage{graphicx}
\usepackage{amsmath}
\usepackage{amssymb}
\usepackage{amsthm}
\usepackage{comment}

\usepackage[dvipsnames]{xcolor}
\usepackage{booktabs}
\usepackage{colortbl}
\usepackage{footnote}
\usepackage{accents}
\usepackage{multirow}
\usepackage{makecell}
\usepackage{subcaption}
\usepackage{enumitem} \setlist{nosep}
\usepackage{thmtools}
\usepackage{bbm}

\usepackage{euscript} % \EuScript{B}
\usepackage{mathrsfs} % \mathscr{A}
\theoremstyle{plain}
\newtheorem{proposition}{\textit{Proposition}}

\newtheorem{assumption}{\textit{Assumption}}

\theoremstyle{remark}

\newcommand{\textbi}[1]{\textbf{\textit{#1}}}
\newcommand{\TODO}[1]{\textbf{\color{red}[TODO: #1]}}
\definecolor{Green}{rgb}{0.0, 0.5372, 0.2040}

\usepackage[pagebackref=true,breaklinks=true,letterpaper=true,colorlinks,bookmarks=false]{hyperref}

\begin{document}

%%%%%%%%% TITLE
\title{Binding Biometrics with AI Agent Identifiers for Delegation of Authority}
% \texttt{BIND}: \underline{B}iometric \underline{I}dentifiers for \underline{N}on-repudiable  \underline{D}elegation of Authority \\ by Humans to AI Agents
 %\texttt{BIND}: \underline{B}iometric \underline{I}dentifiers for \underline{N}on-repudiable  \underline{D}elegation 
%\\ of Human Authority to AI Agents
%}
% \texttt{BIND}: \underline{B}iometric \underline{I}dentity framework for \underline{N}on-repudiable \underline{D}elegation 

\author{
Joseph Geo Benjamin \quad Anil K. Jain \quad Karthik Nandakumar\\
% \textit{Michigan State University, East Lansing, MI 48824, USA} \\
% {\tt\small {\{benja161, jain, nandakum\}@msu.edu}}
\textit{Michigan State University, MI, USA} \\
{\tt\small\textcolor{Green}{\textbf{\{benja161, jain, nandakum\}@msu.edu}}}
}

\maketitle
\thispagestyle{empty}

\begin{abstract}
The proliferation of agentic artificial intelligence (AI) systems has raised serious questions about the accountability for tasks performed by AI agents. Ideally, an AI agent must not be allowed to perform critical tasks without explicit authorization by a human operator. Since biometric recognition is one of the most reliable approaches for authenticating individuals, it has the potential to enable authenticated delegation of authority to AI agents. In this work, we present a framework called \texttt{\textbf{BIND}}, which leverages ideas from the field of biometric cryptosystems, to securely bind biometric data of the human user to the AI agent identity (ID) and authority scope (task-specific constraints) at the time of agent authorization. This token/identifier can be presented by the AI agent to an Identity Auditor, who simultaneously performs biometric authentication and recovers the agent ID and scope, thereby enabling real-time user authentication and establishing a non-repudiable proof of human control and delegation of authority. We also provide a practical implementation of the proposed \texttt{\textbf{BIND}} framework based on face features extracted using standard deep neural network models. To facilitate this implementation, we propose a feature adaptation module that transforms real-valued feature embeddings into fixed-length binary representations suitable for a fuzzy commitment construct based on turbo error correcting codes. Experiments demonstrate the practical feasibility of the proposed face cryptosystem, achieving a True Match Rate of $96\%$ at zero False Match Rate and supporting $1024$-bit agent tokens. 
\end{abstract}

% \begin{IEEEkeywords}
% Agentic AI
% \end{IEEEkeywords}

\section{Introduction}

%Earlier iterations of conversational AI were primarily limited to interpreting user inputs and generating responses such as explanations, recommendations, or code snippets.
Smart cities are propelled by Internet of Things (IoT) and these connected devices are now being increasingly controlled by Large Language Model (LLM)-powered AI agents~\cite{agentiot_rivkin2024aiot_smarthome}.  For instance, AI agents can optimize building energy systems~\cite{agentiot_ly2025smart_buildingoper} through live sensor streams and API-based tool invocation, and orchestrate real-time power grid operations across urban electrical infrastructure~\cite{agentiot_jin2025gridmind}. Agentic AI refers to autonomous systems capable of acting on behalf of humans~\cite{multiagent_du2024survey}. While traditional software systems execute deterministic logic and predefined routines, agentic AI systems often exhibit adaptive and non-deterministic behavior, making decisions based on evolving context and goals \cite{agents_yang2025adoption}. Recent advancements such as the Model Context Protocol (MCP)~\cite{url-MCP} have enabled AI agents to invoke tools, access resources, and act in real-world systems on behalf of humans.
Moreover, emerging protocols like Agent2Agent (A2A)~\cite{url-A2A} enable agents to collaborate, delegate sub-tasks, and form multi-step command chains within and across organizational boundaries. 
%These agentic AI systems are increasingly embedded in complex, distributed workflows involving high-stakes objectives. For instance, 

This growing adoption of agentic AI in smart city applications introduces many fundamental challenges. One of them is the need to maintain persistent and verifiable authorization by a human operator across delegated tasks carried out by AI agents~\cite{AIdeleg_south2025position}. This is critical to prevent rogue agents or individuals from acting maliciously and disrupting autonomous smart city ecosystems. This is also necessary for accountability, trust, and regulatory enforcement by linking agent actions to real-world entities that can be subject to existing legal frameworks~\cite{AIid_chan2025infrastructure}. A recent technical report from the OpenID Foundation~\cite{openid_south2025agenticidentity} highlights that existing identity frameworks are only suitable for relatively static single-entity interactions and cannot meet the emerging needs of agentic AI ecosystems.  In particular, better strategies are required for delegation of authority, identity propagation, and multi-agent coordination.

\begin{figure*}[t!]
    \centering
    \includegraphics[width=1\linewidth]{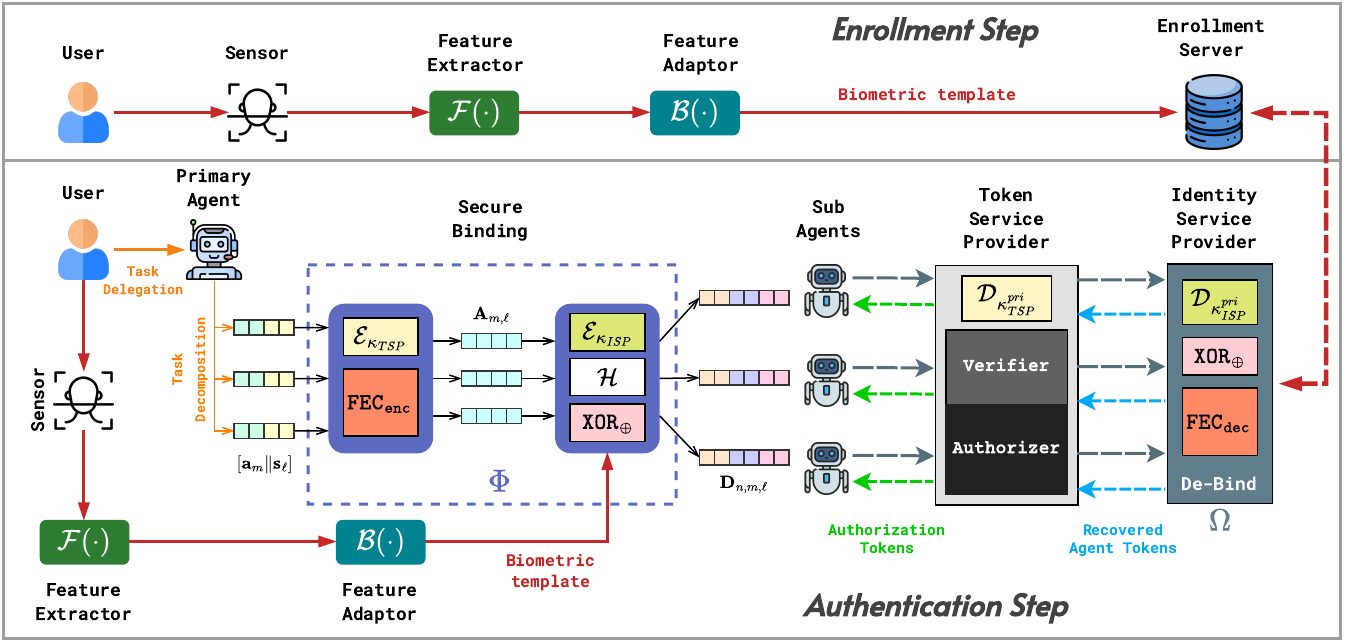}
    \caption{The flow diagram shows a user delegating a task to a primary agent, which decomposes the task into subtasks and generates task-specific Scope and Agent IDs. These identifiers are bound to the user’s biometric credentials and delegated to the corresponding sub-agents for their respective subtasks. Each sub-agent requests authorization from the TSP, which forwards the delegation token to the ISP for biometric de-binding and verification. Upon successful token recovery, the token is returned to the TSP, which validates the delegation and authorizes the requested action.}
    \label{fig:flowdiagram}
\end{figure*}

%% ---- Conventional authorization 
Modern identity frameworks typically isolate authentication and authorization functions. Identity authentication protocols like OpenID Connect (OIDC)~\cite{openid_chadwick2022verifybale_identity} establish user identity, while authorization frameworks like OAuth~\cite{oauth2_rfc6749} enable delegated authorization through scoped access tokens, supporting ``on-behalf-of'' workflows.
Existing mechanisms for controlled token transformation in delegated systems, such as OAuth 2.0 Token Exchange~\cite{oauth2TE_RFC8693}, provide a foundation for managing authorization across agent boundaries through token exchange, scope reduction, audience restriction, and limited token lifetimes. However, as delegation propagates across chains of tools and agents, alignment between individual actions and the authorizer's original intent becomes increasingly difficult to verify, since conventional token-based mechanisms only encode what a token permits and \textit{fails to capture which agent may act, when, or under what task context}~\cite{AIdeleg_south2025position}. Thus, agentic AI systems are exposed to issues like privilege escalation, confused deputy attacks, and unintended cross-boundary data disclosure because access trust is transitively extended across agents\cite{framework_ji2026taming}.

%Thus, conventional mechanisms are well-suited to provide authorization guarantees in client–server settings, but do not support per-agent credential scoping in multi-hop chains, runtime re-authorization as tasks evolve, or cross-agent audit trails that preserve traceability to the original authorization.

%% ---- Challenges in AI agents
Emerging frameworks such as Delegation Capability Tokens~\cite{framework_tomavsev2026intelligent} and IETF OAuth Identity Chaining~\cite{url-IETF_identity_chaining} extend the above foundations to enable structured, cryptographically verifiable authority transfer which can be used for multi-agent orchestration.
Agent communication protocols such as MCP~\cite{url-MCP} and A2A~\cite{url-A2A} are also progressively standardizing authentication and delegation mechanisms to address operational gaps across trust boundaries. For instance, in multi-agent workflows, \textit{OAuth access tokens may be forwarded between agents beyond the delegation boundary} originally authorized by the user, particularly in asynchronous pipelines. However, the core problem is that \textbf{\textit{no human-agent binding mechanism is enforced at the time of token issuance}}~\cite{AIdeleg_south2025position}. As authority transfers between agents, the link between actions, agents, and the authorizer weakens, making it unclear whether the user, orchestrator, or sub-agent bears responsibility for a given outcome~\cite{AIid_tallam2026propogation}. Hence, there is an urgent need to address the issue of \textbf{\textit{who}} remains accountable as delegation propagates. Thus, Identity Binding~\cite{AIid_chan2025infrastructure} can be characterized as: \emph{``Anchoring every delegated action to the originating human identity across all hops in the delegation chain to ensure accountability."}

%One challenge is characterization of IDs for AI agents, which must uniquely identify agent instances while being persistent across dynamic, ephemeral instances, encode delegation context and behavioural scope\cite{AIid_chan2024ids}.

%% ---- How Biometrics fits in
Biometric recognition has long been used as a verifiable, non-repudiable mechanism for establishing user identity~\cite{biom_clarke1994human,biombook_jain2011introduction}. While tokens or credentials generated by digital systems are inherently transferable and replayable, biometric signals are intrinsically bound to the human and can provide higher assurance of liveness and human presence. In this work, we leverage biometrics not only for user identity verification, but also as a mechanism for identity binding between humans and AI agents. Whenever the human operator delegates a new task to an agent or sub-agent, the agent/sub-agent ID and the task scope are encoded in the form of an \emph{Agent Token}, which is ``cryptographically'' bound to the biometric data of the user to generate a \emph{Delegation Token}. At the time of task execution, the agent presents the delegation token to an \emph{Identity Auditor}, who validates this token based on the enrolled biometric template of the user to retrieve the agent token, thereby providing a non-repudiable mechanism for both user authentication and agent authorization. This ensures that agent actions remain verifiably linked to the originating human, while also making the information encoded within the agent token resistant to tampering. This proposed framework is henceforth referred to as \underline{B}iometrics-based \underline{I}dentifiers for \underline{N}on-repudiable  \underline{D}elegation of Authority  (\texttt{\textbf{BIND}}).

%This can be achieved by invoking well-known primitives in the field of biometric cryptosystems \cite{}.

We also propose a practical instantiation of the \texttt{\textbf{BIND}} framework based on face biometric modality. Towards this end, we develop a novel \textit{feature adaptation} module that transforms real-valued face embeddings generated using existing deep neural network models into arbitrary-length binary representations through order-statistic quantization. The feature adaptation step is carefully designed to achieve a reasonable trade-off between preserving discriminability and achieving the desired invariance properties. The resulting binary face representations are securely bound to the agent token using the well-known \textit{fuzzy commitment} construct, enabled by turbo error correction coding schemes. The main contributions of this work are two-fold:
\begin{itemize}[leftmargin=*]
    \item We propose a framework called \texttt{\textbf{BIND}} for identity binding between humans and AI agents, which enables simultaneous user authentication and auditable delegation of authority by humans to AI agents.  
    \item We demonstrate the practical feasibility of the \texttt{\textbf{BIND}} framework based on face biometrics by developing a feature adaptation technique that bridges the representation gap between existing deep face embeddings and the well-known fuzzy commitment scheme. The best setting achieves a $96\%$ TMR at zero-FMR on the CFP-FF dataset and supports agent tokens of size $1024$ bits.
\end{itemize}

\section{Related Works}
\noindent \textbf{Agentic AI Protocols:}
Recent efforts in agentic AI have led to the development of a suite of open standards and protocols to enable interoperability, tool integration, communication, and identity among autonomous agents. Protocols such as MCP \cite{url-MCP} standardize agent access to external tools, APIs, and data sources. In parallel, A2A protocol \cite{url-A2A} facilitates peer-to-peer communication, capability discovery, and task delegation among agentic AI systems. The Agent Communication Protocol (ACP)~\cite{url-ACP} focuses on structured messaging between agents using REST/HTTP-based communication and multimodal support. The Agent Network Protocol (ANP)~\cite{url-ANP} provides network standards for decentralized discovery and collaboration, while the Open Agent Specification (Agent Spec)~\cite{url-AgentSpec} addresses consistent capability and metadata definitions.
%For specialized applications, commerce-focused protocols such as the Universal Commerce Protocol (UCP)~\cite{url-UCP_universalCommerceProtocol} and the x402 Microtransaction Protocol \cite{url-x402Foundation} define standards for agent-mediated transactions and lightweight payments.
Human-centered interaction is supported by the Agent-to-Human Protocol (A2H)~\cite{A2H_liang2025a2h} and Agent-to-UI (AG-UI / A2UI) \cite{url-googleA2UI} formalizes real-time agent–user engagement. The Agent Identity Protocol (AIP)\cite{url-AIP_agentidentityprotocol} is an emerging standard for agent identity and trust, providing verifiable identities, authentication, and policy enforcement. By enabling cryptographic identities, signed actions, and fine-grained authorization, it mitigates security risks from unconstrained agent behavior. Collectively, these protocols form a layered interoperability stack for agentic AI ecosystems \cite{survey_ehtesham2025agenticai_protocols}. However, the problem of secure delegation of authority by humans to agents has not been addressed in any of these existing protocols.

\noindent \textbf{Biometric Cryptosystems:}
Biometric template protection (BTP) refers to techniques that secure stored biometric representations against identity leakage and reconstruction attacks~\cite{biom_template-jain2005Challenges}. 
It is broadly achieved through two complementary approaches: \textit{cancelable biometrics}, which relies on non-invertible transformations applied to raw templates, and \textit{biometric cryptosystems}, which use cryptographic binding or key generation to protect biometric data~\cite{biom_protection-rathgeb2011survey}.
Schemes such as fuzzy commitment~\cite{fuzzy_juels1999fuzzy_commit} and fuzzy vault~\cite{fuzzy_juels2002fuzzy_vault} bind a cryptographic key to biometric data, allowing key recovery only when a sufficiently similar biometric is presented without exposing the underlying template. Several practical implementations of these constructs based on different biometric modalities have been proposed in the literature \cite{fuzzycommit_rane2013secure_survey,fuzzy_hao2006combining_iris, fuzzy_van2006face, fuzzy_jin2016biometric_finger}. Cancelable biometrics is also a key paradigm for template protection, designed to satisfy irreversibility, revocability, unlinkability, and performance preservation. Early approaches focused on transforming biometric features using mathematical mechanisms to prevent direct reconstruction~\cite{biom_template-ratha2001enhancing,hash_goh2003biohash,cancel_rathgeb2013alignment_bloom,ssk_biom-mai2020secureface}. More recent methods rely on deep learning ~\cite{cancel_kim2021ironmask} and generative AI models \cite{cancel_ghafourian2023otb_morphing,cancel_wang2025pp_FaceSupport} to achieve similar goals. To the best of our knowledge, the application of biometrics for human control of AI agents has not been explored well in the literature.

\section{\texttt{BIND} Framework}

\subsection{Problem Formulation}

Let $\mathcal{U} = \{\mathbf{u}_1, \mathbf{u}_2, \cdots, \mathbf{u}_N\}$ be the set of humans interacting with AI agents $\mathcal{A} = \{\mathbf{a}_1, \mathbf{a}_2, \cdots, \mathbf{a}_M\}$ in an organization/system environment. Here, $\mathbf{u}_n$ represents the user ID, $\mathbf{a}_m$ denotes the agent ID, $N$ is the total number of users, and $M$ is the total number of agents. We assume that each human is enrolled with an identity service provider (ISP) based on their biometric traits. Let $\mathbf{b}_n$ represent the biometric template of user $\mathbf{u}_n$ stored in the ISP database. The agents in $\mathcal{A}$ are capable of carrying out various tasks on behalf of users and can interact with tools or resources or even another agent internally or in the external environment. For example, the internal environment may consist of user devices such as laptops or mobile phones, whereas the external environment may correspond to application platforms such as banking systems or travel services.

Let $\mathbf{s}_\ell$ represent the agent scope for the specified task (including the authority delegated by the user to the agent). In this work, our first goal is to design a binding function $\Phi$ that can create a secure and verifiable delegation token $\mathbf{D}_{n,m,\ell}$ for a specific agent $\mathbf{a}_m$ with defined scope $\mathbf{s}_\ell$ for future authorization and auditing purposes. This delegation token is obtained by binding freshly acquired biometric data $\mathbf{\tilde{b}}_n$ from user $\mathbf{u}_n$ with the agent ID and scope as follows:
\begin{equation}
    \begin{aligned}
       \mathbf{D}_{n,m,\ell} := \Phi(\mathbf{\tilde{b}}_n, \mathbf{a}_{m},  \mathbf{s}_{\ell})
    \end{aligned}
\end{equation}

In our formulation, successful authentication and delegation of authority is defined as the correct recovery of the agent ID and scope from the delegation token using only the corresponding user biometric template as the key. The recovered agent details can be used to establish credentials and issue session tokens within standard OAuth flows~\cite{oauth2_rfc6749}, but the exact formulation of these later mechanisms is beyond the scope of this work. Let $\Omega$ denote the de-binding function that can recover the agent ID and scope only upon successful biometric authentication by a valid user. When the authentication fails, the recovery should fail and a garbage value ($\perp$) must be returned by $\Omega$. Thus, 

\begin{equation}
    \begin{aligned}
     \Omega(\mathbf{D}_{n,m,\ell},\mathbf{b}_*) :=
    \begin{cases}
    (\mathbf{a}_{m},\mathbf{s}_\ell)\;, & \text{if } d(\mathbf{b}_*,\mathbf{\tilde{b}}_n) \leq \eta \\
    \perp, & \text{otherwise},
    \end{cases}
    \end{aligned}
\end{equation}

\noindent where $d$ is an appropriate distance metric between two biometric representations and $\eta$ is the decision threshold. The goal of this work is to design the binding ($\Phi$) and de-binding ($\Omega$) functions to satisfy the above requirements as well as withstand various adversarial threats described below.

\subsection{Threat Model}

The generation and lifecycle management of identifiers within AI systems exposes multiple attack surfaces that can be exploited by adversarial actors, including malicious humans and rogue AI agents within the system. Broadly, these threats can be categorized into three types ~\cite{AIid_chan2024ids, AIdeleg_south2025position}.
\begin{enumerate}[leftmargin=*]
    \item \textbi{Tampering:} An adversary alters the delegation token while it is transmitted from the user to the agent and then to the service provider, potentially de-linking the user-agent identity binding or altering the scope. 
    \item  \textbi{Identity Spoofing:} An adversary fabricates a fraudulent user or agent ID and falsely presents the delegation token as originating from a trusted user or agent, thereby impersonating a valid delegation session or shifting responsibility to another entity.
    \item \textbi{Instance Spoofing:} An adversary takes a delegation token issued by a valid user and applies it to a different, unauthorized delegation instance involving a different agent or user, thereby misusing the original delegation. 
\end{enumerate}

\noindent All these threats will eventually result in repudiation claims and erode the trust in the entire framework. Thus, any identity binding framework must be robust against these threats.

\subsection{Proposed Solution}

In this work, we leverage the fuzzy commitment construct \cite{fuzzy_juels1999fuzzy_commit, ssk-biom_nandakumar2010fingerprint} from the biometric template protection (BTP) literature to achieve user-agent identity binding. Fuzzy commitment is typically used for BTP in the following way. The enrolled biometric template is bound to an error correcting codeword (indexed by a secret key) to obtain a secure sketch. During authentication, the secure sketch along with the biometric query is used to recover the codeword, and hence the secret key. The authentication is successful if the secret key can be recovered correctly. Since it is computationally hard to disentangle the template or the codeword from the secure sketch, the template remains protected.

In the proposed identity binding framework, we make two key changes to the fuzzy commitment construct. First, the agent ID and scope are encoded in the form of an \textbf{\textit{Agent Token}}, which acts as the secret key that must be protected. Second, rather than using the secure sketch to protect the biometric template, we generate the secure sketch by binding the codeword indexed by the agent token with the biometric query. This secure sketch acts as the \textbf{\textit{Delegation Token}} provided by the human to the AI agent. The biometric template is stored in the ISP database in plaintext form. If template protection is also desired, other BTP approaches, such as cancelable biometrics, can be applied to secure the enrolled biometric template. During execution of the task by an agent, the agent presents the delegation token to the ISP. The ISP uses the biometric template in conjunction with the delegation token (secure sketch) to recover the agent token. Successful recovery of the agent token signifies both successful authentication of the user and delegation of authority by the user to the agent.

Now, We illustrate the proposed identity binding framework through a simple workflow.
Suppose that user $\mathbf{u}_n$ wants to assign a task to agent $\mathbf{a}_m$ in the system. The agent interprets the assigned task and scopes out the privileges required to perform the task. This scope is encoded in the form of binary representation $\mathbf{s}_\ell$. Note that $\mathbf{s}_\ell$ may include time limits or access rights required for executing the assigned task as well as other auxiliary information (e.g., one-time pads to ensure token liveliness, etc.)\cite{AIid_chan2024ids} and the details of this scope encoding process are beyond our scope. The agent ID $\mathbf{a}_m$ and scope encoding $\mathbf{s}_\ell$ are presented to the user, who generates the agent token $\mathbf{A}_{m,\ell}$ as follows:

\begin{equation}
    \mathbf{A}_{m,\ell} = \mathtt{{FEC}_{enc}}\left(\mathcal{E}_{\kappa_{\mathrm{TSP}}}\left([\mathbf{a}_m||\:\mathbf{s}_\ell]\right)\right),
\end{equation}
where $\mathtt{{FEC}_{enc}}$ is the encoder module of a forward error correction (FEC) scheme that generates a valid codeword indexed by its input $x$, $\mathcal{E}$ is the encryption function of a public key cryptosystem, $\kappa_{\mathrm{TSP}}$ is the public key of a token service provider (TSP) such as OAuth, and $||$ denotes the concatenation operation. Next, a fresh biometric sample is acquired from the user $\mathbf{u}_n$ and the extracted features $\mathbf{\tilde{b}}_n$ are used to generate the secure sketch as:

\begin{equation}
    \mathbf{S}_{n,m,\ell} = \mathbf{\tilde{b}}_n \oplus \mathbf{A}_{m,\ell}
\end{equation}
where $\oplus$ denotes the exclusive-OR operation. Finally, the delegation token is obtained as:

\begin{equation}
    \mathbf{D}_{n,m,\ell} = \Phi(\mathbf{\tilde{b}}_n, \mathbf{a}_{m},  \mathbf{s}_{\ell}) := \left[\mathbf{S}_{n,m,\ell} || \:\mathcal{E}_{\kappa_{\mathrm{ISP}}}(\mathbf{u_n}) \:||\: \mathbf{H}_{m,\ell}\right],
\end{equation}
where $\kappa_{\mathrm{ISP}}$ is the public key of the ISP and $\mathbf{H}_{m,\ell}$ is the cryptographic hash of the input to the FEC encoder, generated by $\mathcal{H}$. The delegation token is provided to the agent $\mathbf{a}_m$ for task execution. For a complex task that requires multiple agents to work together, the primary (planner) agent can request agent/task-specific delegation tokens from the user and all these delegation tokens can be created together. For adaptive scenarios where tasks are assigned to sub-agents on the fly, sub-agents must revert back multiple times to the user for delegation of authority. Although this could be a limitation, this inefficiency can be mitigated by requiring explicit delegation by the user only for critical tasks.

During task execution, the agent $\mathbf{a}_m$ presents the delegation token to the ISP, who decrypts the user identity $\mathbf{u}_n$ with function $\mathcal{D}$ using its private key ${\kappa^\mathrm{pri}_{\mathrm{ISP}}}$ and retrieves the stored biometric template $\mathbf{b}_n$ corresponding to the user ID. The ISP performs authentication as follows:

\begin{equation}
    \Omega(\mathbf{D}_{n,m,\ell},\mathbf{b}_n) :=  \mathtt{{FEC}_{dec}}\left(\mathbf{b}_n \oplus \mathbf{S}_{n,m,\ell}\right),
\end{equation}
where $\mathtt{{FEC}_{dec}}$ is the corresponding decoder module of the FEC scheme. The above error correction decoding will be successful if and only if $\mathbf{b}_n$ and $\mathbf{\tilde{b}}_n$ are sufficiently close, i.e., Hamming distance between the template and query binary biometric representations is less than the error correction capability of the selected FEC scheme. In this case, the correct value of $\mathcal{E}_{\kappa_{\mathrm{TSP}}}\left([\mathbf{a}_m||\:\mathbf{s}_\ell]\right)$ will be recovered, which can be verified by computing its hash and comparing with $\mathbf{H}_{m,\ell}$. If the authentication is successful, the ISP logs the user ID $\mathbf{u}_n$ along with $\mathbf{H}_{m,\ell}$ to facilitate future audits. Thus, the ISP also plays the role of an identity auditor. 

The ISP forwards the value of $\mathcal{E}_{\kappa_{\mathrm{TSP}}}\left([\mathbf{a}_m||\:\mathbf{s}_\ell]\right)$ to the TSP, which decrypts this information with decryption function $\mathcal{D}$ using the key ${\kappa^{\mathrm{pri}}_{\mathrm{TSP}}}$ to obtain the agent ID and scope in plaintext form and issues the appropriate session credentials to the agent to execute the task. Furthermore, the TSP logs the agent ID and scope for future audits. Thus, any action taken by AI agents can be traced by the TSP to the specific agent ID, which in turn can link back to the ISP to determine the human who authorized the action. In summary, the proposed framework enables the creation of a biometric identifier for non-repudiable delegation of authority by humans to AI agents. 

\section{Face-based Implementation of \texttt{BIND}}

We now present a practical implementation of the \texttt{\textbf{BIND}} framework based on the face biometric modality. Specifically, we employ Turbo codes as the FEC scheme for fuzzy commitment. To meet the invariance requirements of the Turbo code instantiation used in this work, we propose a feature adaptation technique that transforms face embeddings generated from well-known deep neural network models into binary representations of arbitrary length. 

\subsection{Error Correction Codes}
Let $\mathbf{x} \in \{0,1\}^k$ denote an message sequence of length $k$ bits. Forward error correction codes systematically append redundant parity bits to the original message, thereby encoding the source sequence into a higher-dimensional codeword space. The resulting redundancy enables a decoder to reconstruct the original message $\mathbf{x}$ even from a noisy codeword. A standard Turbo Code~\cite{ecc_berrou1993TurboCode} consists of two parallel-concatenated recursive systematic convolutional encoders coupled through a pseudo-random interleaver. The output codeword length is $r=3k$ for rate$=1/3$ and $r=2k$ for rate$=1/2$. Then, the codeword $\mathbf{c}$ is given by:

\begin{equation}
    \mathbf{c} = \mathtt{TC_{enc}} \left(\mathbf{x}\right) \quad \text{where}\quad \mathbf{c} \in \{0,1\}^r
\end{equation}

Let $\mathbf{y}$ denote the received noisy vector obtained from the codeword $\mathbf{c}$ after a stochastic process. This process is typically modeled as signal transmission with additive noise $\mathbf{e}$ introduced onto $\mathbf{c}$ such that $\mathbf{y} = \mathbf{c}\,\oplus\,\mathbf{e}$. \textit{However, in our setting, $\mathbf{e}$ captures intra-user biometric variability}. Specifically, it denotes bit-level discrepancies between enrolled biometric template and the query biometric features introduced during the secure sketch recovery process.

The Turbo decoder $\mathtt{TC_{dec}}(\cdot)$ performs iterative decoding using two soft-input/soft-output decoders connected via the interleaver. The recovered codeword is obtained as:
\begin{equation}
\mathbf{\hat{x}} = \mathtt{TC_{dec}}(\mathbf{y}),
\end{equation}
%At each iteration, the decoders exchange log-likelihood ratio information to progressively refine the bit estimates. After $T$ decoding iterations, 
where $\mathbf{y}$ denotes the corrupted codeword. In this work, we use the off-the-shelf Turbo Code provided by NVIDIA’s Sionna library\footnote{\url{https://github.com/NVlabs/sionna}} for error correction coding \cite{lib_hoydis2022sionna}. A given Turbo code can only correct up to $\tau\%$ of bit errors, which is the operating characteristic specific to the code design. Moreover, $\tau$ is a function of the code rate, and lower rates allow for larger $\tau$. Thus, the bit errors caused due to biometric intra-user variability must lie within the $\tau$ bound.

\subsection{Face Feature Adaptation}

Let $\mathbf{f}_n \in \mathbb{R}^D$ be a face feature vector (embedding) extracted from a face image of user $\mathbf{u}_n$ using an existing deep neural network model $\mathcal{F}$ \cite{face_deng2019arcface}. Typically, these face embeddings are real-valued $D$-dimensional vectors, and the similarity between two face embeddings is computed based on cosine similarity. In order to ensure compatibility with the fuzzy commitment construct, these face embeddings must be converted into binary representations of length $r$ bits, where $r$ is determined by the selected FEC scheme. 
Moreover, the proportion of bit errors between representations extracted from two biometric samples belonging to the same user must be less than $\tau$. In contrast, the bit error rate for representations coming from different users must be greater than $\tau$. Thus, there is a need for feature adaptation techniques $\mathcal{B}:\mathbb{R}^D \rightarrow \{0,1\}^r$ that transform $\mathbf{f}_n$ to $\mathbf{b}_n$, where $\mathbf{b}_n = \mathcal{B}(\mathbf{f}_n)$, while satisfying the above constraints.

\begin{figure}
    \centering
    \includegraphics[width=1\linewidth]{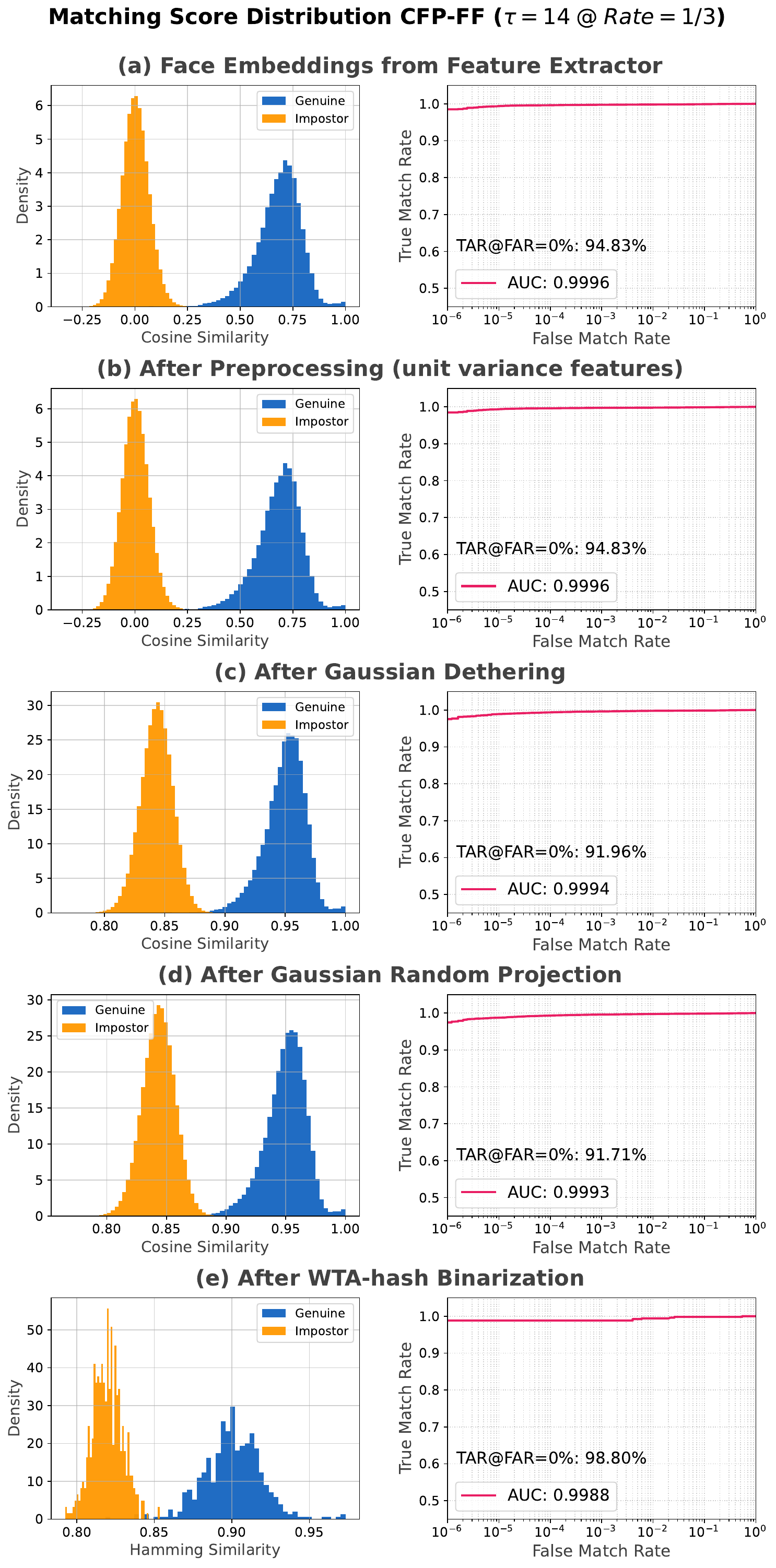}
    \caption{Effect of each step in the Feature Adaptation and Binarization pipeline. Each row shows the resulting distribution shift. The first column shows genuine and impostor scores, and the second shows the ROC curve (log scale). Results indicate that discriminability is preserved without performance degradation.}
    \label{fig:steps-effect}
\end{figure}

We introduce the following feature adaptation method that reduces biometric variability (lowering the bit error rate) while preserving the separability between genuine and impostor pairs. Feature adaptation involves three steps:
\begin{enumerate}[leftmargin=*]
    \item \textbi{Gaussian Dithering:} Superimposes a scaled standard Gaussian random vector onto biometric embeddings. This reduces the effective Hamming distance between biometric templates used for secure sketch, thereby brings the intra-user biometric variability under the error correction capability of the FEC scheme.
    \item \textbi{Gaussian Random Projection:} Projects the signal into a higher-dimensional isotropic space of $B$-dimension. In addition to increasing the template dimensionality from $D$ to any arbitrary $B$, this approach also improves invertibility hardness, as studied in IoM-hashing \cite{cancel_lsh_jin2017ranking_IoM}.
    \item \textbi{WTA-Hash Binarization:} Applies a non-linear, rank-based binarization that improves separation between genuine and impostor distributions by enhancing quantization stability under small perturbations. The non-linear nature of the transformation further increases resistance to inversion attacks.
\end{enumerate}

\begin{assumption}\label{th:normal_embedding}
The biometric feature vectors $\mathbf{f}_n$ are assumed to follow a Gaussian distribution with zero mean and covariance $\sigma_f^2 \mathbf{I}$, where $\sigma_f^2$ depends on the feature extractor $\mathcal{F}$, \ie $\mathbf{f}_n \sim \mathcal{N}(\mathbf{0}, \sigma_f^2 \mathbf{I})$

\end{assumption}

% As a preprocessing step, we perform unit normalization and scale the feature vectors by $\sqrt{D}$. Since the variance of a unit-norm vector is approximately $1/D$, this transformation results in features with approximately unit variance, making them suitable for subsequent analysis. 
We preprocess feature vectors by unit-normalizing and scaling them by $\sqrt{D}$. Since unit-norm vectors have variance of approximately $1/D$, this produces approximately unit-variance features for subsequent analysis.
The feature transformation involves a user-specific \textit{Transformation Key} $\mathbf{K}_n=\{\mathbf{W}_n,\mathbf{g}_n,\Pi_n^{[1]},\Pi_n^{[2]}\}$. 
This transformation key acts as the second factor of authentication, which typically needs to be stored by the user and must be presented during authentication. Moreover, this key $\mathbf{K}_n$ provides cancelability to the enrolled biometric template. 
The overall transformation process can be represented by the following function:

\begin{equation}
\mathbf{b}_n=\xi(\mathbf{\widehat{f}}_n,\mathbf{K}_n), \quad\text{where}\quad \mathbf{\widehat{f}}_n = \frac{\mathbf{f}_n}{\Vert\mathbf{f}_n\Vert}\cdot \sqrt{D}
\end{equation}

\noindent Feature transformation process $\xi$ first projects the perturbed biometric feature into an intermediate representation as follows:
\begin{equation}
\mathbf{q}_n = \mathbf{W}_n\left(\mathbf{\widehat{f}}_n+\lambda\mathbf{g}_n\right),
\end{equation}
where $\mathbf{g}_n \in \mathbb{R}^{D}$ denotes a Gaussian dithering vector, $\mathbf{W}_n \in \mathbb{R}^{2\times r\times D}$ is a Gaussian random projection matrix, and $\mathbf{q}_n \in \mathbb{R}^{2\times r}$ is the intermediate representation. The intermediate representation is intentionally generated at twice the target template length, as required by the subsequent WTA-hashing operation. Next, the WTA-hashing is applied to the projected features. Independent random permutations are first applied to the two projection sets, after which the corresponding elements are compared to generate the final $r$-bit binary template $\mathbf{b}_n$:

\begin{equation}
\mathbf{b}_n= \mathbbm{1}\!\left[
\Pi_n^{[1]}(\mathbf{q}_n^{[1]})
<
\Pi_n^{[2]}(\mathbf{q}_n^{[2]})
\right],
\end{equation}
where $\mathbbm{1}[\cdot]$ is the indicator function applied element-wise. Here, $\Pi_n^{[1]}$ and $\Pi_n^{[2]}$ denote two independently sampled permutations over the index set $\{1,\cdots,r\}$. The permutations randomize the ordering of the projected coefficients prior to comparison, thereby introducing additional randomness into the generated binary template.

\subsection{Determining the Optimal Dithering Factor $\big(\lambda\big)$}

During feature transformation, the Gaussian dithering step can shift the similarity in the real-valued biometric representation $\mathbf{f}$ such that the final template error is bounded within $\tau$. The amount of shift in $\mathbf{f}$ can be controlled by the $\lambda$ value. For a chosen operating threshold $\psi$ derived from the genuine–impostor distribution of $\mathbf{f}$, the corresponding selection of $\lambda$ based on $\tau$ can be formulated by chaining the similarity shifts introduced at each step of the feature adaptation process. This allows the derivation of a single expression that relates $\lambda$ to $\tau$ under the constraint imposed by the operating threshold $\psi$.

\begin{proposition}[Effect of Gaussian Dithering]
Let $\mathbf{x}, \mathbf{y} \sim \mathcal{N}(\mathbf{0}, \mathbf{I}_D)$ be two high-dimensional vectors with cosine similarity $\psi$. Let $\mathbf{g} \sim \mathcal{N}(\mathbf{0}, \mathbf{I}_D)$ be a random noise vector independent of both $\mathbf{x}$ and $\mathbf{y}$. Consider the perturbed vectors
$\mathbf{x}' = \mathbf{x} + \lambda \mathbf{g}, \quad \mathbf{y}' =\mathbf{y} + \lambda \mathbf{g}$.
Then, for sufficiently large $D$, the cosine similarity $\psi^g$ between $\mathbf{x}'$ and $\mathbf{y}'$ satisfies:
\[
\psi^g \approx \frac{\psi + \lambda^2}{1 + \lambda^2}.
\]
\end{proposition}

\begin{proposition}[Effect of Gaussian Random Projection]
Let $\mathbf{x}, \mathbf{y} \sim \mathcal{N}(\mathbf{0}, \mathbf{I}_D)$ be two high-dimensional vectors with cosine similarity $\psi^g$. Let $\mathbf{W} \in \mathbb{R}^{r \times D}$ be a random projection matrix independent of both $\mathbf{x}$ and $\mathbf{y}$, where each entry satisfies $\mathbf{W}^{[i,j]} \sim \mathcal{N}(0,1)$. Consider the projected vectors
$\mathbf{x}' = \mathbf{Wx}, \quad \mathbf{y}' = \mathbf{Wy}$.
Then, for sufficiently large $D$ and $r$, the cosine similarity $\psi^p$ between $\mathbf{x}'$ and $\mathbf{y}'$ satisfies:
\[
\psi^p \approx \psi^g.
\]
\end{proposition}

\begin{proposition}[Effect of Pairwise WTA-Hashing]
Let $\mathbf{x}, \mathbf{y}$ be two high-dimensional vectors that are jointly Gaussian with i.i.d. coordinate pairs having correlation $\psi^p$, i.e. $(\mathbf{x}^{[k]}, \mathbf{y}^{[k]}) \sim \mathcal{N}(0, \Sigma)$, and $\mathtt{corr}(\mathbf{x}^{[k]}, \mathbf{y}^{[k]})=\psi^p$.
So, their cosine similarity converges to $\psi^p$ as $m$ becomes large. 

Let $\Pi_1, \Pi_2$ denote random perturbation operators acting on the coordinates of the vectors. Define the binary embeddings
$\mathbf{w} = \Pi_1(\mathbf{x}) < \Pi_2(\mathbf{x})$ and $\mathbf{v} = \Pi_1(\mathbf{y}) < \Pi_2(\mathbf{y})$
where the comparison is applied element-wise to produce binary vectors.
Then, for sufficiently large $r$, the Hamming similarity $\psi^h$ between $\mathbf{w}$ and $\mathbf{v}$ satisfies
\[
\psi^h \approx 1 - \frac{\arccos(\psi^p)}{\pi}.
\]
\end{proposition}

\noindent Proofs of these propositions are provided in the supplementary material, along with the derivation for Eq.~\ref{eq:lambda-finder}. For a distribution with operating similarity threshold $\psi$ using an error-correcting code capable of correcting a fraction $\tau$ of bit errors, the parameter $\lambda$ is chosen as follows:

\begin{equation}\label{eq:lambda-finder}
\lambda = \sqrt{\frac{\mathtt{cos}(\pi \tau) - \psi}{1 - \mathtt{cos}(\pi \tau)}}.
\end{equation}

\begin{figure}
    \centering
    \includegraphics[width=1\linewidth]{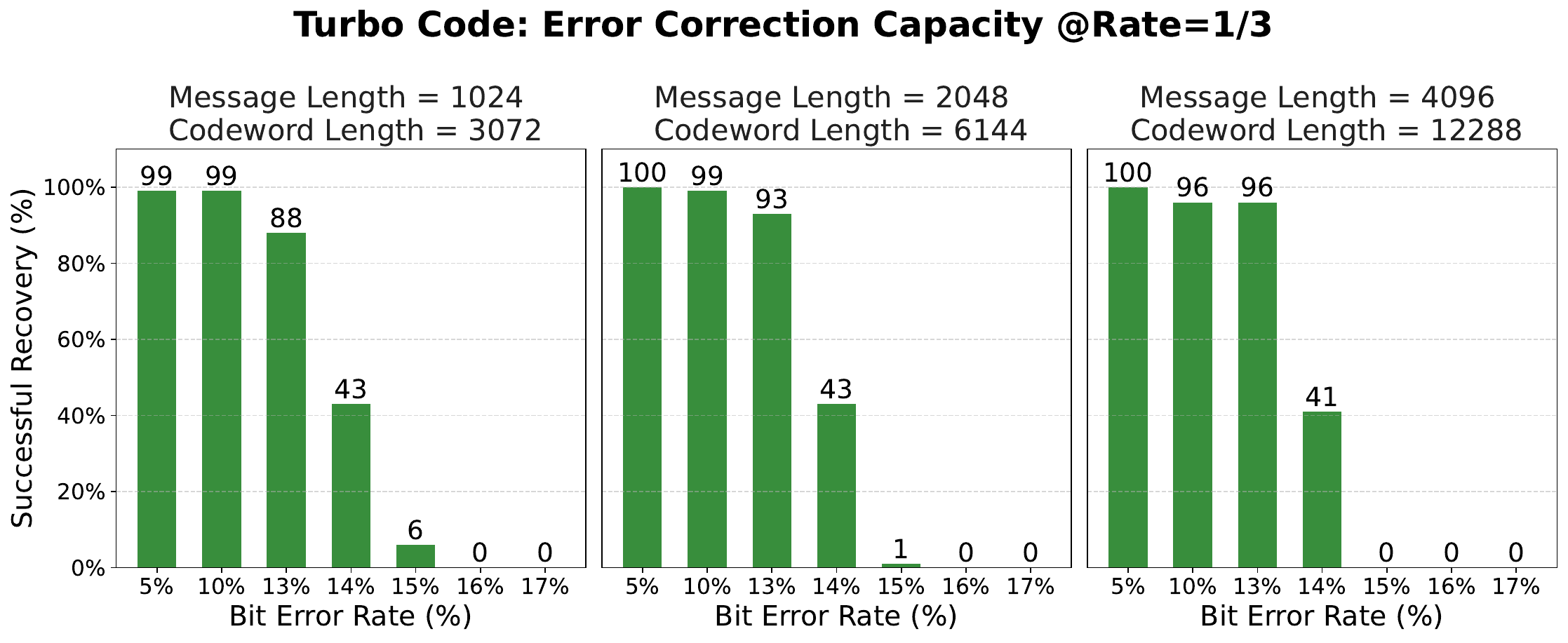} %\\[1ex]
    \includegraphics[width=1\linewidth]{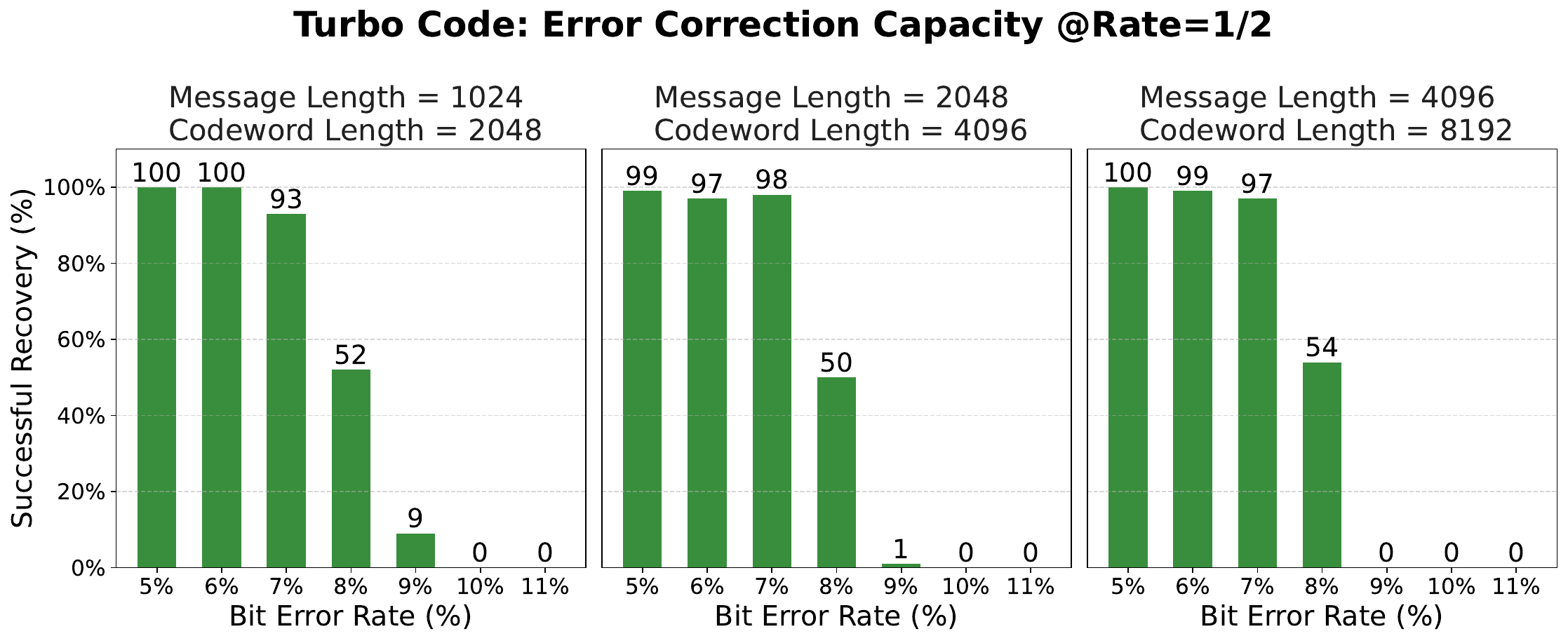}
    \caption{Biometric features show intra-user variability, so a fuzzy commitment scheme must tolerate such deviations. This figure demonstrates Turbo Code’s ability to correct random bit-flip errors across two code rates ($1/2$, $1/3$) and three message lengths ($k$). The error-free decoding threshold remains consistent for a given code rate across different $k$, enabling variable-sized Agent Tokens.}
    \label{fig:ecc-capacity}
\end{figure}

\definecolor{highlightcolor}{HTML}{B2DFDB}
\newcommand{\hlc}[1]{\cellcolor{highlightcolor}#1}

\definecolor{bitcolor}{HTML}{1A237E}
\newcommand{\bitc}[1]{\textcolor{bitcolor}{\textbf{#1}}}

\begin{table*}[t]
\centering
\small
\setlength{\tabcolsep}{1pt} %col gap
\renewcommand{\arraystretch}{1.2} %row gap
\resizebox{\textwidth}{!}{
\begin{tabular}{ccc|cccccc| cccccc| cccccc}
\toprule[1pt]
\midrule
& & & \multicolumn{6}{c|}{\textbf{CFP-FF}} & \multicolumn{6}{c|}{\textbf{LFW-a}} & \multicolumn{6}{c}{\textbf{Multi-PIE}} \\
& & \bitc {Agent Token Length} & \multicolumn{2}{c}{\bitc{1024}} & \multicolumn{2}{c}{\bitc{2048}} & \multicolumn{2}{c|}{\bitc{4096}} & \multicolumn{2}{c}{\bitc{1024}} & \multicolumn{2}{c}{\bitc{2048}} & \multicolumn{2}{c|}{\bitc{4096}} & \multicolumn{2}{c}{\bitc{1024}} & \multicolumn{2}{c}{\bitc{2048}} & \multicolumn{2}{c}{\bitc{4096}} \\
& & \textbf{} & TMR$(\uparrow)$ & FMR$(\downarrow)$ & TMR$(\uparrow)$ & FMR$(\downarrow)$ & TMR$(\uparrow)$ & FMR$(\downarrow)$ & TMR$(\uparrow)$ & FMR$(\downarrow)$ & TMR$(\uparrow)$ & FMR$(\downarrow)$ & TMR$(\uparrow)$ & FMR$(\downarrow)$ & TMR$(\uparrow)$ & FMR$(\downarrow)$ & TMR$(\uparrow)$ & FMR$(\downarrow)$ & TMR$(\uparrow)$ & FMR$(\downarrow)$\\
\midrule
\midrule
\multirow{6}{*}{\rotatebox{90}{Rate=1/3}} 
& \multirow{3}{*}{\rotatebox{90}{$\tau=14$}}
& IResNet101/Arc & 93.40 & 0.00 & 89.80 & 0.00 & 85.00 & 0.00 & 91.31 & 0.00 & 87.56 & 0.00 & 78.27 & 0.00 & 86.35 & 0.00 & 83.53 & 0.00 & 73.49 & 0.00 \\
&& IResNet101/Ada & 92.60 & 0.20 & 93.40 & 0.00 & 83.80 & 0.00 & \hlc93.69 & \hlc0.12 & \hlc92.32 & \hlc0.00 & \hlc85.00 & \hlc0.00 &  \hlc96.39 &  \hlc0.00 &  \hlc89.56 &  \hlc0.00 &  \hlc85.94 &  \hlc0.00 \\
&& KPRPE-ViTb/Ada & \hlc96.00 & \hlc0.00 &  \hlc93.20 &  \hlc0.00 &  \hlc84.80 &  \hlc0.00 & 93.27 & 0.12 & 88.39 & 0.00 & 84.29 & 0.00 & 89.16 & 0.00 & 88.35 & 0.00 & 75.50 & 0.00 \\
\cmidrule(l){2-21}
& \multirow{3}{*}{\rotatebox{90}{$\tau=15$}}
& IResNet101/Arc & 90.00 & 0.00 & 81.00 & 0.00 & 77.60 & 0.00 & 86.49 & 0.00 & 80.83 & 0.00 & 72.50 & 0.00 & 76.71 & 0.00 & 76.31 & 0.00 & 64.26 & 0.00 \\
&& IResNet101/Ada & 89.60 & 0.00 & 81.40 & 0.00 & 75.60 & 0.00 & 90.83 & 0.00 & 85.83 & 0.00 & 78.69 & 0.00 & 85.54 & 0.00 & 82.73 & 0.00 & 70.28 & 0.00 \\
&& KPRPE-ViTb/Ada & 93.00 & 0.20 & 87.60 & 0.00 & 80.60 & 0.00 & 88.81 & 0.00 & 83.99 & 0.00 & 75.30 & 0.00 & 85.14 & 0.00 & 76.31 & 0.00 & 66.67 & 0.00 \\

\midrule[1pt]
\multirow{6}{*}{\rotatebox{90}{Rate=1/2}} 
& \multirow{3}{*}{\rotatebox{90}{$\tau=8$}}
& IResNet101/Arc & 96.60 & 3.80 & 96.60 & 0.40 & 94.60 & 0.00 & 95.95 & 2.62 & 96.07 & 0.65 & 94.94 & 0.06 & 93.57 & 2.41 & 94.38 & 0.00 & 93.98 & 0.00 \\
&& IResNet101/Ada & 96.80 & 2.00 & 96.60 & 0.20 & 96.80 & 0.00 & 96.90 & 4.35 & 96.67 & 0.89 & 96.55 & 0.12 & 97.59 & 4.42 & 95.98 & 0.40 & 95.58 & 0.00 \\
&& KPRPE-ViTb/Ada & 96.40 & 7.80 & 96.40 & 1.20 & 96.40 & 0.20 & 96.25 & 2.14 & 96.25 & 0.54 & 96.13 & 0.12 & 96.79 & 1.20 & 91.57 & 0.40 & 94.38 & 0.00 \\

\cmidrule(l){2-21}
& \multirow{3}{*}{\rotatebox{90}{$\tau=9$}}
& IResNet101/Arc & 91.40 & 0.20 & 90.40 & 0.00 & 89.20 & 0.00 & 88.57 & 0.06 & 88.69 & 0.00 & 86.37 & 0.00 & 85.14 & 0.00 & 80.32 & 0.00 & 80.32 & 0.00 \\
&& IResNet101/Ada & 92.60 & 0.20 & 90.60 & 0.00 & 89.60 & 0.00 & 93.04 & 0.36 & 92.50 & 0.00 & 91.85 & 0.00 & 90.36 & 0.00 & 89.16 & 0.00 & 88.35 & 0.00 \\
&& KPRPE-ViTb/Ada & 92.40 & 0.60 & 92.80 & 0.00 & 91.00 & 0.00 & 90.00 & 0.06 & 89.23 & 0.00 & 89.70 & 0.00 & 83.13 & 0.00 & 84.74 & 0.00 & 82.73 & 0.00 \\
\midrule
\bottomrule[1pt]
\end{tabular}
}
\caption{Performance of the proposed face-based implementation of \texttt{\textbf{BIND}} across three datasets. Agent Token Length ($k$) is measured in bits; the corresponding biometric template length is either $3k$ or $2k$, depending on the coding rate ($1/3$ or $1/2$, respectively). Rows highlighted in \colorbox{highlightcolor}{green} indicate the best TMR-FMR trade-off. Ideally, best setting should achieve an FMR of zero while maximizing TMR.}
\label{tab:results}
\end{table*}
\section{Experimental Results} 

\subsection{Dataset and Feature Extraction}
We evaluate the proposed method on three standard facial recognition datasets: LFW-a \cite{lfwa_wolf2008descriptor} captures unconstrained variations in pose and illumination; CFP-FF \cite{cfp_sengupta2016} evaluates identity matching under frontal pose; and Multi-PIE \cite{multipie_gross2010soft} provides controlled variations in pose, expression, and illumination.
Features are extracted using three independently trained models. Two models are based on the IResnet101 architecture and use ArcFace \cite{face_deng2019arcface} and AdaFace \cite{face_kim2022adaface} as loss functions, while the third model uses a KPRPE-based transformer encoder architecture \cite{face_kim2024kprpe} trained use the AdaFace loss. All models are trained on WebFace4M\cite{webface_zhu2021webface260m}. Each model produces discriminative face embeddings with a dimensionality of $D=512$. All pretrained models and preprocessing were obtained from the CVLface library\footnote{\url{https://github.com/mk-minchul/CVLface}}.
All results reported in the paper are based on the worst-case assumption that the \textit{transformation key} $(\mathbf{K}_n)$ of a user is not kept secret (\textit{\textbf{i.e., a stolen key scenario}}). Hence, $\mathbf{K}_n$ does not provide any additional entropy that could increase the discriminability of the biometric templates. 
% When $\mathbf{K}_n$ is kept secret, higher representation entropy is expected, improving TMR/FMR performance.

\subsection{Main Results}

We evaluate the proposed implementation of the \texttt{\textbf{BIND}} framework by varying three key parameters and the results are summarized in Table \ref{tab:results}. First, we consider two code rates, namely, the default code rate of $1/3$ and a higher rate of $1/2$ obtained through puncturing. Within each code rate, we consider two values of decoding thresholds $\tau$, one ensuring zero recovery $(15\%, 9\%)$ for rates $(1/3, 1/2)$, respectively, and another allowing partial recovery $(14\%, 8\%)$ for rates $(1/3, 1/2)$, respectively. These thresholds are identified from Fig.~\ref{fig:ecc-capacity}. Finally, we consider three message lengths ($k = 1024, 2048, 4096$ bits) to account for varying sizes of agent tokens ($r = 3k$ and $2k$ for rates $1/3$ and $1/2$, respectively). The True Match Rate (TMR) and False Match Rate (FMR), expressed as a \%, are evaluated for each experimental setting. Some key insights from Table \ref{tab:results} are:

\begin{itemize}[leftmargin=*]
    \item Since a higher code rate has reduced error correction capability, we need larger $\lambda$ values to decrease the intra-class variability. While this improves the TMR marginally, the inter-class variability also gets reduced, thereby increasing the FMR significantly. In scenarios where non-repudiation is critical, the configuration with a code rate of $1/3$ and $0$ FMR should be preferred.
    \item Lowering the decoding threshold $\tau$ to allow partial recovery results in higher TMR for the same code rate. Although the zero recovery setting yields a lower TMR compared to the partial recovery setting, it ensures that there are no false accepts.
    \item Increasing the message length ($k$) typically increases the intra-class variability because the entropy of the underlying biometric features remains the same, but the Gaussian transformation introduces more variability. Thus, an increase in message length generally lowers the TMR without significantly affecting the FMR.
    \item The loss function used to train the feature extraction model has an effect on the discriminability of face feature embeddings. Thus, it can be observed that the ArcFace loss function yields a lower TMR compared to both the AdaFace models.
\end{itemize}
Overall, the AdaFace-based implementations with a default code rate of $1/3$ and the partial recovery setting for $\tau$ ($=14$) give the best performance (high TMR with almost $0$ FMR) over all three message lengths as highlighted in Table \ref{tab:results}.

\begin{comment}
    
While he higher code rate achieves slightly better TAR due to the reduced error-correction decoding-threshholdpoints; however, the corresponding FAR is also higher. This behavior may be attributed to the puncturing characteristics of the Turbo Code. 

\paragraph{Effect of $\tau$ Selection}
We compare operating thresholds of the biometric distributions mapped to two decoding-threshhold points for both code rates: the \emph{zero-tolerance} decoding-threshhold(15\% and 9\%), where no error correction occurs, and the \emph{partial-tolerance} decoding-threshhold(14\% and 8\%), where a given code has a partial chance of being fully recovered. 
Naturally, 
Decreasing the error-correction decoding-threshhold to obtain the partial-tolerance setting leads to a higher TAR, but with only a marginal increase in FAR, specifically less than 1\% for the $1/3$ code rate and only a few percentage points for the $1/2$ code rate.

\paragraph{Effect of Increasing Bits}
Increasing the bit count, i.e., the $n \rightarrow m$ mapping ratio increases, both the True Accept Rate (TAR) and the False Accept Rate (FAR) decreases. Typically, at a bit size of 1024, the system achieves a higher TAR, which gradually decreases as the bit size increases.

\paragraph{Effect of Loss Function}

\end{comment}

\subsection{Discussion on Security}

Since the proposed \texttt{\textbf{BIND}} framework is based on the well-studied fuzzy commitment construct, its security properties directly follow from the characteristics of the underlying error correction scheme \cite{fuzzy_juels1999fuzzy_commit}. In particular, the computational complexity of reverse engineering the secure sketch depends upon the inherent entropy of the biometric feature representation \cite{biom-lim2015entropy} as well as the properties of Turbo code (especially its code rate and selected error correction threshold). 
The only modification that we introduce in this work is the feature adaptation step, which imparts some cancelability to the stored biometric template. Given that we do not assume secrecy of the transformation key, the dithering and Gaussian projection steps do not protect against inversion. 
The non-invertibility of the WTA-hash binarization step has already been studied in \cite{cancel_lsh_jin2017ranking_IoM}. Therefore, we defer a more rigorous security analysis of the proposed framework and its specific implementation to future work.

\section{Conclusion}
%% TODO rewrite
%\TODO{Rewrite this based on main flow}
In this work, we address the challenge of authorizing AI agents based on biometric signals that provide strong and irrefutable evidence of user presence. We propose a framework called \texttt{\textbf{BIND}} that tightly binds biometric representations with agent identity and task-specific scope encodings, thereby ensuring that delegated actions remain cryptographically and biologically linked to the authorizing user. Through the innovative use of the fuzzy commitment construct and a novel feature adaptation method, we demonstrate that face biometric data can be reliably used for human-agent binding. Our results highlight the importance of balancing robustness and encoding capacity. Overall, this work provides a step towards more secure agentic AI systems, where delegation of authority by humans to AI agents can be reliably established.

%%-------------------------------------------------------------

\paragraph{Acknowledgment}: This work was partially supported by the Office of Naval Research under Grant No. N00014-24-1-2168.

{\small      
\bibliographystyle{ieee}
\bibliography{reference}
}

\newpage
\appendix
\setcounter{proposition}{0}

\section{Effect of Gaussian Dithering}
\begin{proposition}
Let $\mathbf{x}, \mathbf{y} \sim \mathcal{N}(\mathbf{0}, \mathbf{I}_D)$ be two high-dimensional vectors with cosine similarity $\psi$. Let $\mathbf{g} \sim \mathcal{N}(\mathbf{0}, \mathbf{I}_D)$ be a random noise vector independent of both $\mathbf{x}$ and $\mathbf{y}$. Consider the perturbed vectors
$\mathbf{x}' = \mathbf{x} + \lambda \mathbf{g}, \quad \mathbf{y}' =\mathbf{y} + \lambda \mathbf{g}$.
Then, for sufficiently large $D$, the cosine similarity $\psi^g$ between $\mathbf{x}'$ and $\mathbf{y}'$ satisfies:
\[
\psi^g \approx \frac{\psi + \lambda^2}{1 + \lambda^2}.
\]
\end{proposition}

\begin{proof}
Let the perturbed vectors be
\[
\mathbf{x}' = \mathbf{x} + \lambda \mathbf{g}, \quad
\mathbf{y}' = \mathbf{y} + \lambda \mathbf{g}
\]
The cosine similarity is defined as
\[
\psi^g = \frac{\langle \mathbf{x}', \mathbf{y}' \rangle}{\|\mathbf{x}'\| \, \|\mathbf{y}'\|}
\]

\noindent \textbf{\textit{Expanding the inner product in the numerator}}:
\[
\langle \mathbf{x}', \mathbf{y}' \rangle
= \langle \mathbf{x} + \lambda \mathbf{g}, \mathbf{y} + \lambda \mathbf{g} \rangle
\]
\[
\langle \mathbf{x}', \mathbf{y}' \rangle
= \langle \mathbf{x}, \mathbf{y} \rangle
+ \lambda \langle \mathbf{x}, \mathbf{g} \rangle
+ \lambda \langle \mathbf{y}, \mathbf{g} \rangle
+ \lambda^2 \langle \mathbf{g}, \mathbf{g} \rangle
\]
Using results for high-dimensional Gaussian vectors,
\[
\langle \mathbf{x}, \mathbf{g} \rangle \approx 0, 
\quad \langle \mathbf{y}, \mathbf{g} \rangle \approx 0,
\quad \langle \mathbf{g}, \mathbf{g} \rangle \approx d,
\quad \langle \mathbf{x}, \mathbf{y} \rangle \approx S\cdot d
\]
Hence,
\[
\langle \mathbf{x}', \mathbf{y}' \rangle \approx S\cdot d + \lambda^2\cdot d
\]

\noindent \textbf{\textit{Expanding the norms in the denominator}}:
\[
\|\mathbf{x}'\|^2
= \|\mathbf{x}\|^2 + 2\lambda \langle \mathbf{x}, \mathbf{g} \rangle + \lambda^2 \|\mathbf{g}\|^2
\]
Using results for high-dimensional Gaussian vectors,
\[
\|\mathbf{x}\|^2 \approx d, 
\quad \langle \mathbf{x}, \mathbf{g} \rangle \approx 0, 
\quad \|\mathbf{g}\|^2 \approx d
\]
\[
\|\mathbf{x}'\|^2 \approx d(1 + \lambda^2)
\]
Similarly,
\[
\|\mathbf{y}'\|^2 \approx d(1 + \lambda^2)
\]
Thus,
\[
\Rightarrow \quad
\|\mathbf{x}'\| \approx \|\mathbf{y}'\| \approx \sqrt{d(1+\lambda^2)}
\]

\noindent\textbf{\textit{Substituting these approximations}}:
\[
\psi^g \approx
\frac{d(S + \lambda^2)}{\sqrt{d(1+\lambda^2)} \, \sqrt{d(1+\lambda^2)}}
\]
After simplification \(d(1+\lambda^2)\), we obtain
\[
\psi^g \approx \frac{S + \lambda^2}{1 + \lambda^2}
\]

\end{proof}

\section{Effect of Gaussian Random Projection}
\begin{proposition}
Let $\mathbf{x}, \mathbf{y} \sim \mathcal{N}(\mathbf{0}, \mathbf{I}_D)$ be two high-dimensional vectors with cosine similarity $\psi^g$. Let $\mathbf{W} \in \mathbb{R}^{r \times D}$ be a random projection matrix independent of both $\mathbf{x}$ and $\mathbf{y}$, where each entry satisfies $\mathbf{W}^{[i,j]} \sim \mathcal{N}(0,1)$. Consider the projected vectors
$\mathbf{x}' = \mathbf{Wx}, \quad \mathbf{y}' = \mathbf{Wy}$.
Then, for sufficiently large $D$ and $r$, the cosine similarity $\psi^p$ between $\mathbf{x}'$ and $\mathbf{y}'$ satisfies:
\[
\psi^p \approx \psi^g.
\]
\end{proposition}

\begin{proof}
The cosine similarity after projection is
\[
\psi^p
=
\frac{\langle \mathbf{Wx},\mathbf{Wy}\rangle}
{\|\mathbf{Wx}\|\,\|\mathbf{Wy}\|}
\]

\noindent The terms can be expanded as follows
\[
\langle \mathbf{Wx},\mathbf{Wy}\rangle
=
\mathbf{x}^T\mathbf{W}^T\mathbf{W}\mathbf{y}
\]
\[
\|\mathbf{Wx}\|^2
=
\mathbf{x}^T\mathbf{W}^T\mathbf{W}\mathbf{x}
\]
\[
\|\mathbf{Wy}\|^2
=
\mathbf{y}^T\mathbf{W}^T\mathbf{W}\mathbf{y}
\]

\noindent Since the entries of $\mathbf{W}$ are i.i.d.\ $\mathcal N(0,1)$,
\[
\mathbf{W}^T\mathbf{W}
=
\sum_{i=1}^{m}\mathbf{w}_i\mathbf{w}_i^T
\]
where $\mathbf{w}_i$ denotes the $i$-th row of $\mathbf{W}$.

\noindent The expectation for each row can be written as
\[
\mathbb E[\mathbf{w}_i\mathbf{w}_i^T]
=
\mathbf{I}_d
\]
Thus,
\[
\mathbb E[\mathbf{W}^T\mathbf{W}]
=
m\mathbf{I}_d
\]

\noindent Hence, for sufficiently large $m$,
\[
\mathbf{W}^T\mathbf{W}
\approx
m\mathbf{I}_d
\]

\noindent Substituting into the numerator,
\[
\langle \mathbf{Wx},\mathbf{Wy}\rangle
=
\mathbf{x}^T\mathbf{W}^T\mathbf{W}\mathbf{y}
\approx
m\,\mathbf{x}^T\mathbf{y}
\]

\noindent Similarly,
\[
\|\mathbf{Wx}\|^2
\approx
m\,\|\mathbf{x}\|^2,
\qquad
\|\mathbf{Wy}\|^2
\approx
m\,\|\mathbf{y}\|^2
\]
Therefore the similarity can be written as,
\[
\psi^p
\approx
\frac{m\,\mathbf{x}^T\mathbf{y}}
{\sqrt{m\,\|\mathbf{x}\|^2}\,
 \sqrt{m\,\|\mathbf{y}\|^2}}
=
\frac{\mathbf{x}^T\mathbf{y}}
{\|\mathbf{x}\|\,\|\mathbf{y}\|}
=
S
\]

\noindent Thus, random Gaussian projection approximately preserves cosine similarity,
\[
\psi^p \approx \psi^g
\]
\end{proof}

\section{Effect of Pairwise WTA-Hashing}
\begin{proposition}
Let $\mathbf{x}, \mathbf{y}$ be two high-dimensional vectors that are jointly Gaussian with i.i.d. coordinate pairs having correlation $\psi^p$, i.e. $(\mathbf{x}^{[k]}, \mathbf{y}^{[k]}) \sim \mathcal{N}(0, \Sigma)$, and $\mathtt{corr}(\mathbf{x}^{[k]}, \mathbf{y}^{[k]})=\psi^p$.
So, their cosine similarity converges to $\psi^p$ as $m$ becomes large. 

Let $\Pi_1, \Pi_2$ denote random perturbation operators acting on the coordinates of the vectors. Define the binary embeddings
$\mathbf{w} = \Pi_1(\mathbf{x}) < \Pi_2(\mathbf{x})$ and $\mathbf{v} = \Pi_1(\mathbf{y}) < \Pi_2(\mathbf{y})$
where the comparison is applied element-wise to produce binary vectors.
Then, for sufficiently large $r$, the Hamming similarity $\psi^h$ between $\mathbf{w}$ and $\mathbf{v}$ satisfies
\[
\psi^h \approx 1 - \frac{\arccos(\psi^p)}{\pi}.
\]
\end{proposition}

\begin{proof}
Consider a single bit of the hash. Let $(i,j)$ be the coordinate pair
selected by the random perturbation operators $\Pi_1$ and $\Pi_2$.
The resulting binary bits are
\[
u=\mathbf{1}\{x_i<x_j\},
\qquad
v=\mathbf{1}\{y_i<y_j\}
\]
which can be written equivalently as,
\[
u=\mathbf{1}\{x_j-x_i>0\},
\qquad
v=\mathbf{1}\{y_j-y_i>0\}
\]

\noindent Lets define
\[
\Delta x=x_j-x_i,
\qquad
\Delta y=y_j-y_i
\]

\noindent Assume that the coordinate pairs $(x_i,y_i)$ and $(x_j,y_j)$ are jointly
Gaussian with correlation coefficient equal to the cosine similarity $\rho = \psi^p$.
Since differences of Gaussian random variables remain Gaussian,
$(\Delta x, \Delta y)$ is a bivariate Gaussian vector.
\\

\noindent \textbf{\textit{Showing that Correlation equal to Cosine Similarity}}

\noindent The variance can be written as 
\[
\operatorname{Var}(\Delta x)
=
\operatorname{Var}(x_j)+\operatorname{Var}(x_i)
=
2
\]
and similarly
\[
\operatorname{Var}(\Delta y)=2
\]

Further the covariance will be given by,
\begin{align*}
\operatorname{Cov}(X,Y)
&=
\operatorname{Cov}(x_j-x_i,\; y_j-y_i) \\
&=
\operatorname{Cov}(x_j,y_j)
-\operatorname{Cov}(x_j,y_i) \\
& \quad-\operatorname{Cov}(x_i,y_j)
+\operatorname{Cov}(x_i,y_i)
\end{align*}

\noindent Since different coordinates are independent,
\[
\operatorname{Cov}(x_j,y_i)
=
\operatorname{Cov}(x_i,y_j)
=
0
\]
and therefore
\[
\operatorname{Cov}(\Delta x, \Delta y)
=
\psi^p+\psi^p
=
2\psi^p
\]

\noindent Hence correlation $\rho$ can be written as,
\[
\operatorname{Corr}(\Delta x, \Delta y)
=
\frac{\operatorname{Cov}(\Delta x, \Delta y)}
{\sqrt{\operatorname{Var}(\Delta x)\operatorname{Var}(\Delta x)}}
=
\frac{2\psi^p}{\sqrt{2}\sqrt{2}}
=
S
\]

\noindent \textbf{\textit{Hamming Similarity Equivalence}}

\noindent The bits $u$ and $v$ agree precisely when $\Delta x$ and $\Delta y$ have the same sign.
For a bivariate Gaussian pair with correlation $S$, Sheppard's formula
states that
\[
\mathbb{P}\Bigl(\operatorname{sign}(\Delta x)=\operatorname{sign}(\Delta y)\Bigr)
=
1-\frac{\arccos(\psi^p)}{\pi}
\]

Therefore,
\[
\mathbb{P}(u=v)
=
1-\frac{\arccos(\psi^p)}{\pi}
\]

The Hamming similarity is
\[
\psi^h
=
\frac{1}{m}
\sum_{t=1}^{m}
\mathbf{1}\{u_t=v_t\}
\]
Since the bits are i.i.d. Bernoulli random variables with success
probability
\[
1-\frac{\arccos(\psi^p)}{\pi},
\]
Thus, for sufficiently large $b$, the hamming similiarity can be written as as follows \cite{lsh_charikar2002simhash},
\[
\psi^h
\approx
1-\frac{\arccos(\psi^p)}{\pi}
\]
\end{proof}

\section{Deriving $\lambda$ from Similarity  Contraints}

\noindent Combining the above relations, we obtain the forward model:
\[
\psi^h \approx 1 - \frac{\arccos\!\left(\frac{\psi + \lambda^2}{1 + \lambda^2}\right)}{\pi}.
\]

Then,
\[
\cos\bigl(\pi(1-\psi^h)\bigr)
\approx
\frac{\psi + \lambda^2}{1 + \lambda^2}.
\]

Define
\[
c := \cos\bigl(\pi(1-\psi^h)\bigr).
\]

We solve for $\lambda^2$:
\[
c(1 + \lambda^2) = \psi + \lambda^2,
\]
\[
c + c\lambda^2 = \psi + \lambda^2,
\]
\[
\lambda^2(c - 1) = \psi - c.
\]

Thus,

\[
\lambda^2 = \frac{c - \psi}{1 - c}.
\]

Finally,
\[
\lambda = \sqrt{\frac{c - \psi}{1 - c}},
\quad
\text{where } c = \cos\bigl(\pi(1-\psi^h)\bigr).
\]

This provides a closed-form estimator for the dithering strength $\lambda$ to obtain Hamming similarity $\psi^h$ and the original cosine similarity $S$.

For a given operating threshold $\psi$ indicating Cosine Similarity and decoding error bound $\tau$, which in terms of Hamming Similarity is $1-\tau$. Then formula is written as 

\[
\lambda = \sqrt{\frac{\cos\bigl(\pi\tau\bigr) - \psi}{1 - \cos\bigl(\pi\tau\bigr)}}.
\]

This $\lambda$ is used in the feature adaptation process, ensuring that the biometric representation is compatible with the error-correction code.

\end{document}